\documentclass{article}

\usepackage[final]{neurips_2021}

\usepackage[utf8]{inputenc} 
\usepackage[T1]{fontenc}    
\usepackage[colorlinks,
            linkcolor=red,
            anchorcolor=blue,
            citecolor=blue
            ]{hyperref}
\usepackage{url}            
\usepackage{booktabs}       
\usepackage{amsfonts}       
\usepackage{nicefrac}       
\usepackage{microtype}      
\usepackage[dvipsnames]{xcolor}         
\usepackage{enumitem}
\usepackage{smile}

\newcommand{\qvalue}{Q}
\newcommand{\vvalue}{V}

\newcommand{\reward}{r}
\newcommand{\la}{\langle}
\newcommand{\ra}{\rangle}

\title{Provably Efficient Reinforcement Learning with Linear Function Approximation under Adaptivity Constraints}

\author{%
  Tianhao Wang\thanks{Equal contribution} \\
  Department of Statistics and Data Science\\
  Yale University\\
  New Haven, CT 06511\\
  \texttt{tianhao.wang@yale.edu}\\
  \And 
  Dongruo Zhou$^*$\\
  Department of Computer Science\\
  University of California, Los Angeles\\
  Los Angeles, CA 90095\\
  \texttt{drzhou@cs.ucla.edu}\\
  \And
  Quanquan Gu\\
  Department of Computer Science\\
  University of California, Los Angeles\\
  Los Angeles, CA 90095\\
  \texttt{qgu@cs.ucla.edu}\\
}

\begin{document}

\maketitle

\begin{abstract}
We study reinforcement learning (RL) with linear function approximation under the adaptivity constraint. We consider two popular limited adaptivity models: the batch learning model and the rare policy switch model, and propose two efficient online RL algorithms for episodic linear Markov decision processes, where the transition probability and the reward function can be represented as a linear function of some known feature mapping. In specific, for the batch learning model, our proposed LSVI-UCB-Batch algorithm achieves an $\tilde O(\sqrt{d^3H^3T} + dHT/B)$ regret, where $d$ is the dimension of the feature mapping, $H$ is the episode length, $T$ is the number of interactions and $B$ is the number of batches. Our result suggests that it suffices to use only $\sqrt{T/dH}$ batches to obtain $\tilde O(\sqrt{d^3H^3T})$ regret. For the rare policy switch model, our proposed LSVI-UCB-RareSwitch algorithm enjoys an $\tilde O(\sqrt{d^3H^3T[1+T/(dH)]^{dH/B}})$ regret, which implies that $dH\log T$ policy switches suffice to obtain the $\tilde O(\sqrt{d^3H^3T})$ regret. Our algorithms achieve the same regret as the LSVI-UCB algorithm \citep{jin2020provably}, yet with a substantially smaller amount of adaptivity. We also 
establish a lower bound for the batch learning model, which suggests that the dependency on $B$ in our regret bound is tight.
\end{abstract}

    
    
    
    

\section{Introduction}
Real-world reinforcement learning (RL) applications often come with possibly infinite state and action space, and in such a situation classical RL algorithms developed in the tabular setting are not applicable anymore. A popular approach to overcoming this issue is by applying function approximation techniques to the underlying structures of the Markov decision processes (MDPs). For example, one can assume that the transition probability and the reward are linear functions of a known feature mapping $\bphi:\cS\times\cA\to\RR^d$, where $\cS$ and $\cA$ are the state space and action space, and $d$ is the dimension of the embedding. This gives rise to the so-called linear MDP model \citep{yang2019sample,jin2020provably}.
Assuming access to a generative model, efficient algorithms under this setting have been proposed by \citet{yang2019sample} and \citet{lattimore2020learning}. For online finite-horizon episodic linear MDPs, \citet{jin2020provably} proposed an LSVI-UCB algorithm that achieves $\Tilde{O}(\sqrt{d^3H^3T})$ regret, where $H$ is the planning horizon (i.e., length of each episode) and $T$ is the number of interactions.

However, all the aforementioned algorithms require the agent to update the policy in every episode. In practice, it is often unrealistic to frequently switch the policy in the face of big data, limited computing resources as well as inevitable switching costs. Thus one may want to batch the data stream and update the policy at the end of each period. For example, in clinical trials, each phase~(batch) of the trial amounts to applying a medical treatment to a batch of patients in parallel. The outcomes of the treatment are not observed until the end of the phase and will be subsequently used to design experiments for the next phase. 
Choosing the appropriate number and sizes of the batches is crucial to achieving nearly optimal efficiency for the clinical trial. This gives rise to the limited adaptivity setting, which has been extensively studied in many online learning problems including prediction-from-experts (PFE) \citep{kalai2005efficient,cesa2013online}, multi-armed bandits (MAB) \citep{arora2012online,cesa2013online} and online convex optimization \citep{jaghargh2019consistent, chen2020minimax}, to mention a few. Nevertheless, in the RL setting, learning with limited adaptivity is relatively less studied. \citet{bai2019provably} introduced two notions of adaptivity in RL, \emph{local switching cost} and \emph{global switching cost}, that are defined as follows
\begin{align}\label{eq: switching cost}
    N_{\text{local}} = \sum_{k=1}^{K-1}\sum_{h=1}^H\sum_{s\in\cS} \ind\{\pi_h^k(s)\neq \pi_h^{k+1}(s)\}\quad\text{and}\quad N_{\text{global}} = \sum_{k=1}^{K-1} \ind\{\pi^k \neq \pi^{k+1}\},
\end{align}
where $\pi^k = \{\pi_h^k:\cS\to\cA\}_{h\in[H]}$ is the policy for the $k$-th episode of the MDP,  $\pi^k\neq\pi^{k+1}$ means that there exists some $(h,s)\in[H]\times\cS$ such that $\pi_h^k(s)\neq\pi_h^{k+1}(s)$, and $K$ is the number of episodes. Then they proposed a Q-learning method with UCB2H exploration that achieves $\Tilde{O}(\sqrt{H^3SAT})$ regret with $O(H^3SA\log(T/(AH))$ local switching cost for tabular MDPs, but they did not provide tight bounds on the global switching cost. 

In this paper, based on the above motivation, we aim to develop online RL algorithms with linear function approximation under adaptivity constraints. In detail, we consider time-inhomogeneous\footnote{We say an episodic MDP is time-inhomogeneous if its reward and transition probability are different at different stages within each episode. See Definition \ref{assumption-linear} for details.} episodic linear MDPs \citep{jin2020provably} where both the transition probability and the reward function are unknown to the agent. 
In terms of the limited adaptivity imposed on the agent, we consider two scenarios that have been previously studied in the online learning literature \citep{perchet2016batched, abbasi2011improved}: the batch learning model and the rare policy switch model. More specifically, in the batch learning model \citep{perchet2016batched}, the agent is forced to pre-determine the number of batches (or equivalently batch size). 
Within each batch, the same policy is used to select actions, and the policy is updated only at the end of this batch. The amount of adaptivity in the batch learning model is measured by the number of batches, which is expected to be as small as possible. In contrast, in the rare policy switch model \citep{abbasi2011improved}, the agent can adaptively choose when to switch the policy and therefore start a new batch in the learning process as long as the total number of policy updates does not exceed the given budget on the number of policy switches. The amount of adaptivity in the rare policy switch model can be measured by the number of policy switches, which turns out to be the same as the global switching cost introduced in \citet{bai2019provably}.
It is worth noting that for the same amount of adaptivity\footnote{The number of batches in the batch learning model is comparable to the number of policy switches in the rare policy switch model.}, the rare policy switch model can be seen as a relaxation of the batch learning model since the agent in the batch learning model can only change the policy at pre-defined time steps.
In our work,  for each of these limited adaptivity models, we propose a variant of the LSVI-UCB algorithm \citep{jin2020provably}, which can be viewed as an RL algorithm with full adaptivity in the sense that it switches the policy at a per-episode scale. Our algorithms can attain the same regret as LSVI-UCB, yet with a substantially smaller number of batches/policy switches. This enables parallel learning and improves the large-scale deployment of RL algorithms with linear function approximation.

The main contributions of this paper are summarized as follows:
\begin{itemize}[leftmargin = *]
    \item For the \emph{batch learning} model, we propose an LSVI-UCB-Batch algorithm for linear MDPs and show that it enjoys an $\tilde O(\sqrt{d^3H^3T} + dHT/B)$ regret, where $d$ is the dimension of the feature mapping, $H$ is the episode length, $T$ is the number of interactions and $B$ is the number of batches. Our result suggests that it suffices to use only $\sqrt{T/dH}$ batches, rather than $T$ batches, to obtain the same regret $\tilde O(\sqrt{d^3H^3T})$ achieved by LSVI-UCB \citep{jin2020provably} in the fully sequential decision model. We also prove a lower bound of the regret for this model, which suggests that the required number of batches $\tilde O(\sqrt{T})$ is sharp.
    \item For the \emph{rare policy switch} model, we propose an LSVI-UCB-RareSwitch algorithm for linear MDPs and show that it enjoys an $\tilde O(\sqrt{d^3H^3T[1+T/(dH)]^{dH/B}})$ regret, where $B$ is the number of policy switches. Our result implies that $dH\log T$ policy switches are sufficient to obtain the same regret $\tilde O(\sqrt{d^3H^3T})$ achieved by LSVI-UCB. The number of policy switches is much smaller than that\footnote{The number of policy switches is identical to the number of batches in the batch learning model.} of the batch learning model when $T$ is large. 
\end{itemize}

Concurrent to our work,  \citet{gao2021provably} proposed an algorithm achieving $\tilde O(\sqrt{d^3H^3T})$ regret with a $O(dH\log K)$ global switching cost in the rare policy switch model. They also proved a $\Omega(dH/\log d)$ lower bound on the global switching cost.
The focus of our paper is different from theirs: our goal is to design efficient RL algorithms under a switching cost budget $B$, while their goal is to achieve the optimal rate in terms of $T$ with as little switching cost as possible. On the other hand, for the rare policy switch model, our proposed algorithm (LSVI-UCB-RareSwitch) along its regret bound can imply their results by optimizing our regret bound concerning the switching cost budget $B$.

The rest of the paper is organized as follows. In Section \ref{sec: related work} we discuss previous works related to this paper, with a focus on RL with linear function approximation and online learning with limited adaptivity. In Section \ref{section 3} we introduce necessary preliminaries for MDPs and adaptivity constraints. Sections \ref{sec:batch} and \ref{sec:rare} present our proposed algorithms and the corresponding theoretical results for the batch learning model and the rare policy switch model respectively. In Section \ref{sec: experiment} we present the numerical experiment which supports our theory. Finally, we conclude our paper and point out a future direction in Section \ref{sec: conclusions}.

\noindent\textbf{Notation} 
We use lower case letters to denote scalars and use lower and upper case boldface letters to denote vectors and matrices respectively. 
For any real number $a$, we write $[a]^+ = \max(a,0)$.
For a vector $\xb\in \RR^d$ and matrix $\bSigma\in \RR^{d\times d}$, we denote by $\|\xb\|_2$ the Euclidean norm and define $\|\xb\|_{\bSigma}=\sqrt{\xb^\top\bSigma\xb}$. 
For any positive integer $n$, we denote by $[n]$ the set $\{1,\dots,n\}$. 
For any finite set $A$, we denote by $|A|$ the cardinality of $A$. 
For two sequences $\{a_n\}$ and $\{b_n\}$, we write $a_n=O(b_n)$ if there exists an absolute constant $C$ such that $a_n\leq Cb_n$, and we write $a_n=\Omega(b_n)$ if there exists an absolute constant $C$ such that $a_n\geq Cb_n$. 
We use $\tilde O(\cdot)$ to further hide the logarithmic factors.

\section{Related Works}\label{sec: related work}
\textbf{Reinforcement Learning with Linear Function Approximation}
Recently, there have been many advances in RL with function approximation, especially the linear case. \citet{jin2020provably} proposed an efficient algorithm for the first time for linear MDPs of which the transition probability and the rewards are both linear functions with respect to a feature mapping $\bphi:\cS\times\cA\to\RR^d$. 
Under similar assumptions, different settings (e.g., discounted MDPs) have also been studied in \citet{yang2019sample, du2019good,zanette2020learning,neu2020unifying} and \citet{he2020logarithmic}. A parallel line of work studies linear mixture MDPs (a.k.a. linear kernel MDPs) based on a ternary feature mapping $\psi:\cS\times\cA\times\cS\to\RR^d$ (see \citet{jia2020model, zhou2020provably,cai2020provably,zhou2020nearly}). 
For other function approximation settings, we refer readers to generalized linear model \citep{wang2021optimism}, general function approximation with Eluder dimension \citep{wang2020reinforcement,ayoub2020model}, kernel approximation \citep{yang2020bridging}, function approximation with disagreement coefficients \citep{foster2020instance} and bilinear classes \citep{du2021bilinear}.

\noindent\textbf{Online Learning with Limited Adaptivity} 
As we mentioned before, online learning with limited adaptivity has been studied in two popular models of adaptivity constraints: the batch learning model and the rare policy switch model.

For the \emph{batch learning model}, \citet{altschuler2018online} proved that the optimal regret bound for prediction-from-experts (PFE) is $\tilde O(\sqrt{T\log n})$ when the number of batches $B=\Omega(\sqrt{T\log n})$, and $\min(\tilde O(T\log n/B),T)$ when $B=O(\sqrt{T\log n})$, exhibiting a phase-transition phenomenon\footnote{They call it $B$-switching budget setting, which is identical to the batch learning model.}. Here $T$ is the number of rounds and $n$ is the number of actions. 
For general online convex optimization, \citet{chen2020minimax} showed that the minimax regret bound is $\tilde O(T/\sqrt{B})$. \citet{perchet2016batched} studied batched 2-arm bandits, and \citet{gao2019batched} studied the batched multi-armed bandits (MAB). \citet{dekel2014bandits} proved a $\Omega(T/\sqrt{B})$ lower bound for batched MAB, and \citet{altschuler2018online} further characterized the dependence on the number of actions $n$ and showed that the corresponding minimax regret bound is $\min(\tilde O(T\sqrt{n}/\sqrt{B}),T)$. 
For batched linear bandits with adversarial contexts, \citet{han2020sequential} showed that the minimax regret bound is $\tilde O(\sqrt{dT}+dT/B)$ where $d$ is the dimension of the context vectors. Better rates can be achieved for batched linear bandits with stochastic contexts as shown in \citet{esfandiari2019regret,han2020sequential, ruan2020linear}.

For the \emph{rare policy switch model}, the minimax optimal regret bound for PFE is $O(\sqrt{T\log n})$ in terms of both the expected regret \citep{kalai2005efficient, geulen2010regret,cesa2013online,devroye2015random} and high-probability guarantees \citep{altschuler2018online}, where $T$ is the number of rounds, and $n$ is the number of possible actions. For MAB, the minimax regret bound has been shown to be $\tilde O(T^{2/3}n^{1/3})$ by \citet{arora2012online,dekel2014bandits}. 
For stochastic linear bandits, \citet{abbasi2011improved} proposed a rarely switching OFUL algorithm achieving $\tilde O(d\sqrt{T})$ regret with $\log(T)$ batches. \citet{ruan2020linear} proposed an algorithm achieving $\tilde O(\sqrt{dT})$ regret with less than $O(d\log d\log T)$ batches for stochastic linear bandits with adversarial contexts. 

For episodic RL with finite state and action space, \citet{bai2019provably} proposed an algorithm achieving $\tilde O(\sqrt{H^3SAT})$ regret with $O(H^3SA\log(T/(AH)))$ local switching cost where $S$ and $A$ are the number of states and actions respectively. They also provided a $\Omega(HSA)$ lower bound on the local switching cost that is necessary for sublinear regret. For the global switching cost, \citet{zhang2021reinforcement} proposed an MVP algorithm with at most $O(SA\log(KH))$ global switching cost for time-homogeneous tabular MDPs.

\section{Preliminaries}\label{section 3}

\subsection{Markov Decision Processes}
We consider the time-inhomogeneous episodic Markov decision process, which is denoted by a tuple $M(\cS, \cA, H, \{\reward_h\}_{h\in[H]}, \{\PP_h\}_{h\in[H]})$. Here $\cS$ is the state space (may be infinite), $\cA$ is the action space where we allow the feasible action set to change from step to step, $H$ is the length of each episode, and $\reward_h: \cS \times \cA \rightarrow [0,1]$ is the reward function for each stage $h\in[H]$. At each stage $h\in[H]$, $\PP_h(s'|s,a) $ is the transition probability function which represents the probability for state $s$ to transit to state $s'$ given action $a$. A policy $\pi$ consists of $H$ mappings, $\{\pi_h:\cS\to\cA\}_{h\in[H]}$. For any policy $\pi$, we define the action-value function $\qvalue^\pi_h(s,a)$ and value function $\vvalue^\pi_h(s)$ as follows:
\begin{align}
&\qvalue^\pi_h(s,a) = \reward_h(s,a) + \EE_\pi\bigg[\sum_{i = h}^{H} \reward_i(s_{i}, a_{i})\bigg|s_h = s, a_h = a\bigg], \qquad\vvalue^\pi_h(s) = \qvalue^\pi_h(s,\pi_h(s)),\notag
\end{align}
where $a_i\sim\pi_i(\cdot|s_i)$ and $s_{i+1}\sim\PP_i(\cdot|s_i,a_i)$.
The optimal value function $V^*_h$ and the optimal action-value function $\qvalue^*_h(s,a)$ are defined as $V^*(s) = \sup_{\pi}\vvalue_h^{\pi}(s)$ and $\qvalue^*_h(s,a) = \sup_{\pi}\qvalue_h^{\pi}(s,a)$, respectively.
For simplicity, for any function $\vvalue: \cS \rightarrow \RR$, we denote $[\PP \vvalue](s,a)=\EE_{s' \sim \PP(\cdot|s,a)}\vvalue(s')$. 
In the online learning setting, at the beginning of $k$-th episode, the agent chooses a policy $\pi^k$ and the environment selects an initial state $s_1^k$, then the agent interacts with environment following policy $\pi^k$ and receives states $s_h^k$ and rewards $r_h(s_h^k, a_h^k)$ for $h\in[H]$. To measure the performance of the algorithm, we adopt the following notion of the total regret, which is the summation of suboptimalities between policy $\pi^k$ and optimal policy $\pi^*$:
\begin{definition}\label{def: total regret}
We denote $T=KH$, and the regret $\text{Regret}(T)$ is defined as
\begin{align}
    \text{Regret}(T) 
    &=\sum_{k=1}^K\sbr{V_1^*(s_1^k)-V_1^{\pi^k}(s_1^k)}.\notag
\end{align}
\end{definition}

\subsection{Linear Function Approximation}
In this work, we consider a special class of MDPs called \emph{linear MDPs} \citep{yang2019sample, jin2020provably}, where both the transition probability function and reward function can be represented as a linear function of a given feature mapping $\bphi: \cS \times \cA \rightarrow \RR^d$. Formally speaking, we have the following definition for linear MDPs.
\begin{definition}\label{assumption-linear}
$M(\cS, \cA, H, \{\reward_h\}_{h\in[H]}, \{\PP_h\}_{h\in[H]})$ is called a linear MDP if there exist a \emph{known} feature mapping $\bphi(s,a): \cS \times \cA\rightarrow \RR^d$, \emph{unknown} measures $\{\bmu_h=(\mu_h^{(1)},\cdots,\mu_h^{(d)})\}_{ h\in[H]}$ over $\cS$ and unknown vectors $\{\btheta_h\in\RR^d\}_{h\in[H]}$ with $\max_{h\in[H]}\{\|\bmu_h(\cS)\|_2,\|\btheta_h\|\} \leq \sqrt{d}$, such that the following holds for all $h\in[H]$:
\begin{itemize}
    \item For any state-action-state triplet $(s,a,s') \in \cS \times \cA \times \cS$, $\PP_h(s'|s,a) = \la \bphi(s,a), \bmu_h(s')\ra$.
    
    \item For any state-action pair $(s,a)\in\cS\times\cA$, $r_h(s,a)=\langle\bphi(s,a),\btheta_h\rangle$.
\end{itemize}
Without loss of generality, we also assume that $\|\bphi(s,a)\|_2\leq 1$ for all $(s,a)\in\cS\times\cA.$
\end{definition}

With Definition \ref{assumption-linear}, it is shown in \citet{jin2020provably} that the action-value function can be written as a linear function of the features.
\begin{proposition}[Proposition 2.3, \citealt{jin2020provably}]\label{lemma: linear Q}
For a linear MDP, for any policy $\pi$, there exist weight vectors $\{\wb_h^\pi\}_{h\in[H]}$ such that for any $(s,a,h)\in\cS\times\cA\times[H]$, we have $Q_h^\pi(s,a)=\langle\bphi(s,a),\wb_h^\pi\rangle$. Moreover, we have $\|\wb_h^\pi\|_2\leq 2H\sqrt{d}$ for all $h\in[H]$. 
\end{proposition}

Therefore, with the known feature mapping $\bphi(\cdot,\cdot)$, it suffices to estimate the weight vectors $\{\wb_h^\pi\}_{h\in[H]}$ in order to recover the action-value functions. This is the core idea behind almost all the algorithms and theoretical analyses for linear MDPs.

\subsection{Models for Limited Adaptivity}
In this work, we consider RL algorithms with limited adaptivity. There are two typical models for online learning with such limited adaptivity: \emph{batch learning model} \citep{perchet2016batched} and \emph{rare policy switch model} \citep{abbasi2011improved}. 

For the batch learning model, the agent pre-determines the batch grids $1 = t_1<t_2<\cdots<t_B<t_{B+1} = K+1$ at the beginning of the algorithm, where $B$ is the number of batches. The $b$-th batch consists of $t_b$-th to $(t_{b+1}-1)$-th episodes, and the agent follows the same policy within each batch. The adaptivity is measured by the number of batches. 

For the rare policy switch model, the agent can decide whether she wants to switch the current policy or not. The adaptivity is measured by the number of policy switches, which is defined as
\begin{align*}
    N_{\text{switch}}=\sum_{k=1}^{K-1}\ind\{\pi^k\neq \pi^{k+1}\},
\end{align*}
where $\pi^k\neq\pi^{k+1}$ means that there exists some $(h,s)\in[H]\times\cS$ such that $\pi_h^k(s)\neq\pi_h^{k+1}(s)$. It is worth noting that $N_{\text{switch}}$ is identical to the \emph{global switching cost} defined in  \eqref{eq: switching cost}.

Given a budget on the number of batches or the number of policy switches, we aim to design RL algorithms with linear function approximation that can achieve the same regret as their full adaptivity counterpart, e.g., LSVI-UCB \citep{jin2020provably}.

\section{RL in the Batch Learning Model}\label{sec:batch}

In this section, we consider RL with linear function approximation in the batch learning model, where given the number of batches $B$, we need to pin down the batches before the agent starts to interact with the environment. 

\subsection{Algorithm and Regret Analysis}

We propose LSVI-UCB-Batch algorithm as displayed in Algorithm \ref{alg: batch}, which can be regarded as a variant of the LSVI-UCB algorithm proposed in \citet{jin2020provably} yet with limited adaptivity. Algorithm~\ref{alg: batch} takes a series of batch grids $\{t_1,\dots, t_{B+1}\}$ as input, where the $i$-th batch starts at $t_i$ and ends at $t_{i+1}-1$. LSVI-UCB-Batch takes the uniform batch grids as its selection of grids, i.e., $t_i = (i-1)\cdot\lfloor K/B\rfloor+1, i\in [B]$. By Proposition \ref{lemma: linear Q}, we know that for each $h \in [H]$, the optimal value function $\qvalue_h^*$ has the linear form $\la \bphi(\cdot, \cdot), \wb_h^*\ra$. Therefore, to estimate the $\qvalue_h^*$, it suffices to estimate $\wb_h^*$. At the beginning of each batch, Algorithm \ref{alg: batch} calculates $\wb_h^k$ as an estimate of $\wb_h^*$ by ridge regression (Line \ref{alg:batch_w}). Meanwhile, in order to measure the uncertainty of $\wb_h^k$, Algorithm \ref{alg: batch} sets the estimate $\qvalue_h^k(\cdot, \cdot)$ as the summation of the linear function $\la \bphi(\cdot, \cdot), \wb_h^k\ra$ and a Hoeffding-type exploration bonus term $\Gamma_h^k(\cdot, \cdot)$ (Line \ref{alg:batch_q}), which is calculated based on the confidence radius $\beta$. Then it sets the policy $\pi_h^k$ as the greedy policy with respect to $\qvalue_h^k$. Within each batch, Algorithm \ref{alg: batch} simply keeps the policy used in the previous episode without updating (Line \ref{alg:batch_repeat}). Apparently, the number of batches of Algorithm~\ref{alg: batch} is $B$.

Here we would like to make a comparison between our LSVI-UCB-Batch and other related algorithms. The most related algorithm is LSVI-UCB proposed in \citet{jin2020provably}. 
The main difference between LSVI-UCB-Batch and LSVI-UCB is the introduction of batches. 
In detail, when $B = K$, LSVI-UCB-Batch degenerates to LSVI-UCB. Another related algorithm is the SBUCB algorithm proposed by \citet{han2020sequential}. 
Both LSVI-UCB-Batch and SBUCB take uniform batch grids as the selection of batches. 
The difference is that SBUCB is designed for linear bandits, which is a special case of episodic MDPs with $H=1$. 



\begin{algorithm}[t]
	\caption{LSVI-UCB-Batch}
	\label{alg: batch}
	\begin{algorithmic}[1]
	\REQUIRE Number of batches $B$, confidence radius $\beta$, regularization parameter $\lambda$
	\STATE Set $b \leftarrow 1$, $t_i \leftarrow (i-1)\cdot\lfloor K/B\rfloor+1, i\in [B]$ 
	\FOR{episode $k=1,2,\dots,K$}
	\STATE Receive the initial state $s_1^k$
	\IF{$k = t_{b}$}
	\STATE $b \leftarrow b+1$, $Q_{H+1}^k(\cdot, \cdot)\leftarrow 0$
	\FOR{stage $h=H,H-1,\dots,1$}
	\STATE $\bLambda_h^k\leftarrow \sum_{\tau=1}^{k-1} \bphi(s_h^{\tau},a_h^{\tau}) \bphi(s_h^{\tau},a_h^{\tau})^\top + \lambda \Ib$
	\STATE $\wb_h^k \leftarrow (\bLambda_h^k)^{-1} \sum_{\tau=1}^{k-1} \bphi(s_h^{\tau},a_h^{\tau}) \cdot [r_h(s_h^{\tau},a_h^{\tau})+\max_{a\in\cA} Q_{h+1}^k(s_{h+1}^\tau,a)]$\alglinelabel{alg:batch_w}
	\STATE $\Gamma_h^k(\cdot,\cdot)\leftarrow \beta\cdot [\bphi(\cdot,\cdot)^\top(\bLambda_h^k)^{-1} \bphi(\cdot,\cdot)]^{1/2}$\alglinelabel{alg:batch_g}
	\STATE $Q_h^k(\cdot,\cdot) \leftarrow \min\{\bphi(\cdot,\cdot)^\top\wb_h^k + \Gamma_h^k(\cdot,\cdot), H-h+1\}^+$, $\pi_h^k(\cdot)\leftarrow \argmax_{a\in\cA} Q_h^{k}(\cdot,a)$\alglinelabel{alg:batch_q}
	\ENDFOR
	\ELSE 
	\STATE $Q_h^k\leftarrow Q_h^{k-1}$, $\pi_h^k\leftarrow \pi_h^{k-1}$, $\forall h \in [H]$\alglinelabel{alg:batch_repeat}
	\ENDIF
	\FOR{stage $h = 1 \dots, H$}
	\STATE Take the action $a_h^k\leftarrow \pi_h^k(s_h^k)$, receive the reward $r_h(s_h^k,a_h^k)$ and the next state $s_{h+1}^k$
	\ENDFOR
	\ENDFOR
	\end{algorithmic}
\end{algorithm}

The following theorem presents the regret bound of Algorithm \ref{alg: batch}.

\begin{theorem}\label{thm: regret budget1}
There exists a constant $c>0$ such that for any $\delta\in(0,1)$, if we set $\lambda=1$, $\beta=cdH\sqrt{\log(2dT/\delta)}$, then under Assumption \ref{assumption-linear}, the total regret of Algorithm \ref{alg: batch} is bounded by
\begin{align}
    \text{Regret}(T)\leq& 2H\sqrt{T\log\rbr{\frac{2dT}{\delta}}} + \frac{dHT}{2B\log 2}\log\rbr{\frac{T}{dH}+1}\notag \\
    &\qquad + 4c\sqrt{2d^3H^3T\log\rbr{\frac{2dT}{\delta}}\log\bigg(\frac{T}{dH}+1\bigg)}\notag
\end{align}
with probability at least $1-\delta$.
\end{theorem}

Theorem \ref{thm: regret budget1} suggests that the total regret of Algorithm \ref{alg: batch} is bounded by $\tilde O(\sqrt{d^3H^3T} + dHT/B)$. 
When $B = \Omega(\sqrt{T/dH})$, the regret of Algorithm \ref{alg: batch} is $\tilde O(\sqrt{d^3H^3T})$, which is the same as that of LSVI-UCB in \citet{jin2020provably}. 
However, it is worth noting that LSVI-UCB needs $K$ batches, while Algorithm \ref{alg: batch} only requires $\sqrt{T/dH}$ batches, which can be much smaller than $K$. 

Next, we present a lower bound to show the dependency of the total regret on the number of batches for the batch learning model.

\begin{theorem}\label{thm:lower1}
Suppose that $B \geq (d-1)H/2$. Then for any batch learning algorithm with $B$ batches, there exists a linear MDP such that the regret over the first $T$ rounds is lower bounded by
\begin{align}
    \text{Regret}(T) = \Omega(dH\sqrt{T} + dHT/B).\notag
\end{align}
\end{theorem}
Theorem \ref{thm:lower1} suggests that in order to obtain a standard $\sqrt{T}$-regret, the number of batches $B$ should be at least in the order of $\Omega(\sqrt{T})$, which is similar to its counterpart for batched linear bandits \citep{han2020sequential}. 

\section{RL in the Rare Policy Switch Model}\label{sec:rare}

In this section, we consider the rare policy switch model, where the agent can adaptively choose the batch sizes according to the information collected during the learning process.

\subsection{Algorithm and Regret Analysis}

We first present our second algorithm, LSVI-UCB-RareSwitch, as illustrated in Algorithm \ref{alg: rare switch}. Again, due to the nature of linear MDPs, we only need to estimate $\wb_h^*$ by ridge regression, and then calculate the optimistic action-value function using the Hoeffding-type exploration bonus $\Gamma_h^k(\cdot,\cdot)$ along with the confidence radius $\beta$. Note that the size of the bonus term in $Q_h^k$ is determined by $\bLambda_h^k$. Intuitively speaking, the matrix $\bLambda_h^k$ in Algorithm \ref{alg: rare switch} represents how much information has been learned about the underlying MDP, and the agent only needs to switch the policy after collecting a significant amount of additional information. This is reflected by the determinant of $\bLambda_h^k$, and the upper confidence bound will become tighter (shrink) as $\det(\bLambda_h^k)$ increases. The determinant based criterion is similar to the idea of doubling trick, which has been used in the rarely switching OFUL algorithm for stochastic linear bandits \citep{abbasi2011improved}, UCRL2 algorithm for tabular MDPs \citep{jaksch2010near}, and UCLK/UCLK+ for linear mixture MDPs in the discounted setting \citep{zhou2020provably,zhou2020nearly}.

\begin{algorithm}[t]
	\caption{LSVI-UCB-RareSwitch}
	\label{alg: rare switch}
	\begin{algorithmic}[1]
	\REQUIRE Policy switch parameter $\eta$, confidence radius $\beta$, regularization parameter $\lambda$ 
	\STATE Initialize $\bLambda_h=\bLambda_h^0=\lambda\Ib_d$ for all $h\in[H]$
	\FOR{episode $k=1,2,\dots,K$}
	\STATE Receive the initial state $s_1^k$
	\FOR{stage $h=1, 2, \cdots, H$}
    \STATE $\bLambda_h^k\leftarrow \sum_{\tau=1}^{k-1} \bphi(s_h^{\tau},a_h^{\tau}) \bphi(s_h^{\tau},a_h^{\tau})^\top + \lambda \Ib_d$ \alglinelabel{alg rare switch: Lambda^k}
    \ENDFOR
    \IF{$\exists h\in[H], \det(\bLambda_h^k)>\eta\cdot\det(\bLambda_h)$} \alglinelabel{alg rare switch: log det}
    \STATE $Q_{H+1}^k(\cdot,\cdot)\leftarrow 0$
    \FOR{step $h=H, H-1,\cdots,1$}
    \STATE $\bLambda_h\leftarrow\bLambda_h^k$ \alglinelabel{alg rare switch: Lambda}
    \STATE $\wb_h^k \leftarrow (\bLambda_h^k)^{-1} \sum_{\tau=1}^{k-1} \bphi(s_h^{\tau},a_h^{\tau}) \cdot [r_h(s_h^{\tau},a_h^{\tau})+\max_{a\in\cA} Q_{h+1}^k(s_{h+1}^\tau,a)]$ \alglinelabel{alg rare switch: policy evaluation start}
    \STATE $\Gamma_h^k(\cdot,\cdot)\leftarrow \beta\cdot [\bphi(\cdot,\cdot)^\top (\bLambda_h^k)^{-1} \phi(\cdot,\cdot)]^{1/2}$
	\STATE $Q_h^k(\cdot,\cdot) \leftarrow \min\{\phi(\cdot,\cdot)^\top \wb_h^k + \Gamma_h^k(\cdot,\cdot), H-h+1\}^+$, $\pi_h^k(\cdot) \leftarrow \argmax_{a\in\cA} Q_h^k(\cdot,a)$ \alglinelabel{alg rare switch: policy evaluation end}
    \ENDFOR
    \ELSE
    \STATE $Q_h^k\leftarrow Q_h^{k-1}$, $\pi_h^k\leftarrow\pi_h^{k-1}$, $\forall h\in[H]$ \alglinelabel{alg rare switch: not change policy}
    \ENDIF
	\FOR{stage $h = 1 \dots, H$}
	\STATE Take the action $a_h^k\leftarrow \pi_h^k(s_h^k)$, receive the reward $r_h(s_h^k,a_h^k)$ and the next state $s_{h+1}^k$
	\ENDFOR
	\ENDFOR
	\end{algorithmic}
\end{algorithm}

As shown in Algorithm \ref{alg: rare switch}, for each stage $h\in[H]$ the algorithm maintains a matrix $\bLambda_h$ which is updated at each policy switch (Line \ref{alg rare switch: Lambda}). For every $k\in[K]$, we denote by $b_k$ the episode from which the policy $\pi_k$ is computed. This is consistent with the one defined in Algorithm \ref{alg: batch} in Section \ref{sec:batch}. At the start of each episode $k$, the algorithm computes $\{\bLambda_h^k\}_{h\in[H]}$ (Line \ref{alg rare switch: Lambda^k}) and then compares them with $\{\bLambda_h\}_{h\in[H]}$ using the determinant-based criterion (Line \ref{alg rare switch: log det}). The agent switches the policy if there exists some $h\in[H]$ such that $\det(\bLambda_h^k)$ has increased by some pre-determined parameter $\eta>1$, followed by policy evaluation (Lines \ref{alg rare switch: policy evaluation start}-\ref{alg rare switch: policy evaluation end}). Otherwise, the algorithm retains the previous policy (Line \ref{alg rare switch: not change policy}). Here the hyperparameter $\eta$ controls the frequency of policy switch, and the total number of policy switches can be bounded by a function of $\eta$. 

Algorithm \ref{alg: rare switch} is also a variant of LSVI-UCB proposed in \citet{jin2020provably}. Compared with LSVI-UCB-Batch in Algorithm \ref{alg: batch} for the batch learning model, LSVI-UCB-RareSwitch adaptively decides when to switch the policy and can be tuned by the hyperparameter $\eta$ and therefore fits into the rare policy switch model.  

We present the regret bound of Algorithm \ref{alg: rare switch} in the following theorem. 
\begin{theorem}\label{thm: regret rare switch}
There exists some constant $c>0$ such that for any $\delta\in(0,1)$, if we set $\lambda=1$,  $\beta=cdH\sqrt{\log(2dT/\delta)}$ and $\eta = \rbr{1+K/d}^{dH/B}$, then the number of policy switches $N_{\text{switch}}$ in Algorithm \ref{alg: rare switch} will not exceed $B$. Moreover, the total regret of Algorithm \ref{alg: rare switch} is bounded by 
\begin{align}
    &\text{Regret}(T)\leq 2H\sqrt{T\log\rbr{\frac{2dT}{\delta}}} +2c\sqrt{2 d^3H^3T} \cdot\sqrt{\rbr{\frac{T}{dH}+1}^{\frac{dH}{B}}\log\left(\frac{T}{dH}+1\right)\log\rbr{\frac{2dT}{\delta}}}\label{eq: regret rare switch}
\end{align}
with probability at least $1-\delta$.
\end{theorem}

A few remarks are in order.
\begin{remark}
Algorithm \ref{alg: rare switch} needs to update the value of each $\det(\bLambda_h^k)$, and thanks to the special structure of $\bLambda_h^k$, this can be done efficiently by applying the matrix determinant lemma along with the Sherman Morrison formula for efficiently updating each $(\bLambda_h^k)^{-1}$. For simplicity and clarity of the presentation, we do not include these details in the pseudo-code.
\end{remark}

\begin{remark}
By ignoring the non-dominating term, Theorem \ref{thm: regret rare switch} suggests that the total regret of Algorithm \ref{alg: rare switch} is bounded by $\tilde O(\sqrt{d^3H^3T[1+T/(dH)]^{dH/B}})$. Also, if we are allowed to choose $B$, we can choose $B=\Omega(dH\log T)$ to achieve $\tilde O(\sqrt{d^3H^3T})$ regret, which is the same as that of LSVI-UCB in \citet{jin2020provably}. This also significantly improves upon Algorithm \ref{alg: batch} when $T$ is sufficiently large since previously we need $B=\Omega(\sqrt{T/dH})$. 
Our result exhibits a trade-off between the total regret bound and the number of policy switches, i.e., as the adaptivity budget $B$ increases, the regret bound decreases. 
This will also be reflected by the numerical results later in Section \ref{sec: experiment}.
\end{remark}

\begin{remark}
Concurrent to our work, \citet{gao2021provably} proposed an algorithm with $B=\Omega(dH\log T)$ policy switches. Note that $B=\Omega(dH\log T)$ corresponds to choosing $\eta$ to be a constant, which can be viewed as a special case of our algorithm. Their algorithm does not adapt to different values of budget $B$. Also, they did not study the batch learning model (Section~\ref{sec:batch}) which we think is of equally important practical interest.

\end{remark}

\begin{remark} \citet{gao2021provably} established a lower bound, which claims that any rare policy switch RL algorithm suffers a linear regret when $B = \tilde o(dH)$. However, unlike our lower bound for the batch learning model (Theorem \ref{thm:lower1}),  their result does not provide a fine-grained regret lower bound for arbitrary adaptivity constraint $B$. It remains an open problem to establish such kind of lower bound for the rare policy switch model.
\end{remark}



\section{Numerical Experiment}\label{sec: experiment}
In this section, we provide numerical experiments to support our theory. We run our algorithms, LSVI-UCB-Batch and LSVI-UCB-RareSwitch, on a synthetic linear MDP given in Example \ref{eg: mdp1}, and compare them with the fully adaptive baseline, LSVI-UCB \citep{jin2020provably}.

\begin{example}[Hard-to-learn linear MDP, \citealt{zhou2020provably}]\label{eg: mdp1}
Let $d>0$ be some integer and $\delta\in(0,1)$ be a constant. The state space $\cS=\{0,1\}$ consists of two states, and the action space $\cA=\{\pm 1\}^{d-3}$ contains $2^{d-3}$ actions where each action is represented by a $(d-3)$-dimensional vector $\ab$. For each state-action pair $(s,\ab)\in\cS\times\cA$, the feature vector is given by
\begin{align}\label{eq: eg phi}
    \bphi(s,a) = \begin{cases}
    (-\ab^\top, 1-\delta, \delta)^\top & s=0,\\
    (0,\ldots,0,\delta,1-\delta) & s=1.
    \end{cases}
\end{align}
For each $h\in[H]$, let $\bgamma_h\in \{\pm\delta/(d-2)\}^{d-2}$ and define the corresponding vector-valued measure as
\begin{align}\label{eq: eg mu}
    \bmu_h(s) = \begin{cases}
    (\bgamma_h^\top,1,0)^\top & s=0\\
    (-\bgamma_h^\top,0,1)^\top & s=1
    \end{cases}.
\end{align}
Finally, we set $\btheta_h\equiv (0,\ldots,0,-\delta/(1-2\delta),(1-\delta)/(1-2\delta))\in\RR^d$ for all $h\in[H]$.
\end{example}

It is straightforward to verify that the feature vectors in \eqref{eq: eg phi} and the vector-valued measures in \eqref{eq: eg mu} constitute a valid linear MDP such that, for all $\ab\in\cA$ and $h\in[H]$,
\begin{align*}
    r_h(s,\ab) = \ind\{s=1\}, \qquad \PP_h(s'|s,\ab) = \begin{cases}
    1-\delta - \langle\ab,\gamma_h\rangle & (s,s')=(0, 0),\\
    \delta+\langle\ab,\bgamma_h\rangle  & (s,s')=(0,1),\\
    \delta & (s,s')=(1,0),\\
    1-\delta & (s,s')=(1,1).
    \end{cases} 
\end{align*}
In our experiment\footnote{All experiments are performed on a PC with Intel i7-9700K CPU.}, we set $H=10$, $K=2500$, $\delta=0.35$ and $d=13$, thus $\cA$ contains 1024 actions. Now we apply our algorithms, LSVI-UCB-Batch and LSVI-UCB-RareSwitch, to this linear MDP instance, and compare their performance with the fully adaptive baseline LSVI-UCB \citep{jin2020provably} under different parameter settings. In detail, for LSVI-UCB-Batch, we run the algorithm for $B=10,20,30,40,50$ respectively; for LSVI-UCB-RareSwitch, we set $\eta=2,4,8,16,32$. 
We plot the average regret ($\text{Regret}(T)/K$) against the number of episodes in Figure \ref{fig:exp_result}.
In addition to the regret of the proposed algorithms, we also plot the regret of a uniformly random policy (i.e., choosing actions uniformly randomly in each step) as a baseline.

\begin{figure}[t]\label{fig1}
    \begin{subfigure}[LSVI-UCB-Batch]{\label{fig:subfig:1.a}
        \includegraphics[width=0.49\linewidth]{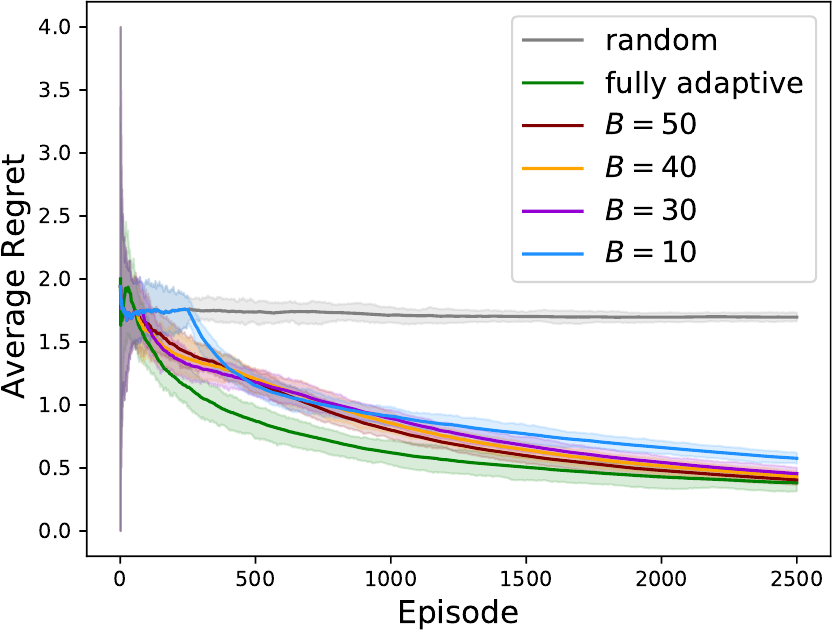}}
    \end{subfigure}
    \begin{subfigure}[LSVI-UCB-RareSwitch]
    {\label{fig:subfig:1.b} 
        \includegraphics[width=0.49\linewidth]{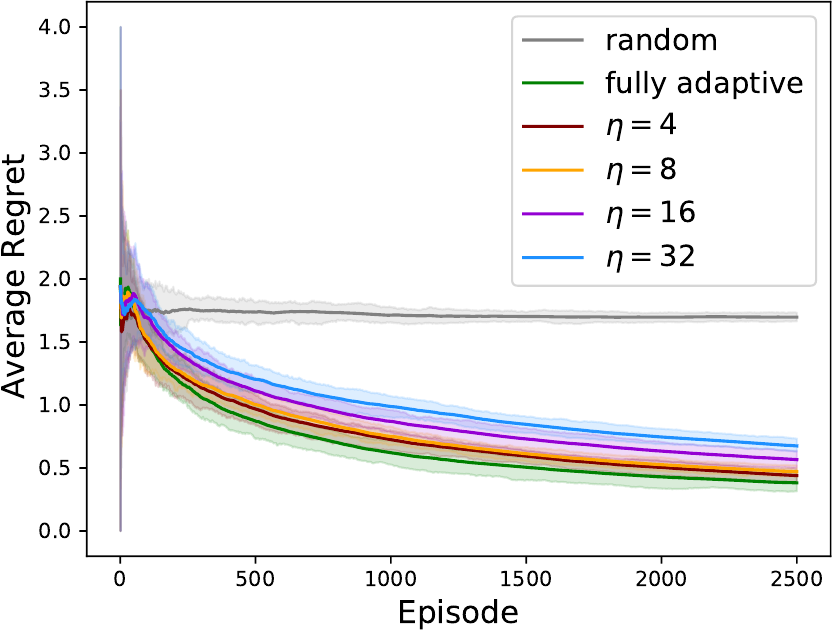}}
    \end{subfigure}
    \caption{Plot of average regret ($\text{Regret}(T)/K$) v.s. the number of episodes. The results are averaged over 50 rounds of each algorithm, and the error bars are chosen to be $[20\%,80\%]$ empirical confidence intervals.}
    \label{fig:exp_result}
\end{figure}

From Figure \ref{fig:exp_result}, we can see that for LSVI-UCB-Batch, when $B\approx \sqrt{K}$, it achieves a similar regret as the fully adaptive LSVI-UCB as it collects more and more trajectories.
For LSVI-UCB-RareSwitch, a constant value of $\eta$ yields a similar order of regret compared with LSVI-UCB as suggested by Theorem~\ref{thm: regret rare switch}. 
By comparing Figure \ref{fig:subfig:1.a} and \ref{fig:subfig:1.b}, we can see that the performance of LSVI-UCB-RareSwitch is consistently close to that of the fully-adaptive LSVI-UCB throughout the learning process, while the performance gap between LSVI-UCB-Batch and LSVI-UCB is small only when $k$ is large.
This suggests a better adaptivity of LSVI-UCB-RareSwitch than LSVI-UCB-Batch, which only updates the policy at prefixed time steps, thus being not adaptive enough.

Moreover, we can also see the trade-off between the regret and the adaptivity level: with more limited adaptivity (smaller $B$ or larger $\eta$) the regret gap between our algorithms and the fully adaptive LSVI-UCB becomes larger.
These results indicate that our algorithms can indeed achieve comparable performance as LSVI-UCB, even under adaptivity constraints. This corroborates our theory.

\section{Conclusions}\label{sec: conclusions}
In this work, we study online RL with linear function approximation under the adaptivity constraints. We consider both the batch learning model and the rare policy switch models and propose two new algorithms LSVI-UCB-Batch and LSVI-UCB-RareSwitch for each setting. We show that LSVI-UCB-Batch enjoys an $\tilde O(\sqrt{d^3H^3T} + dHT/B)$ regret and LSVI-UCB-RareSwitch enjoys an $\tilde O(\sqrt{d^3H^3T[1+T/(dH)]^{dH/B}})$ regret. Compared with the fully adaptive LSVI-UCB algorithm \citep{jin2020provably}, our  algorithms can achieve the same regret with a much fewer number of batches/policy switches. We also prove the regret lower bound for the batch learning learning model, which suggests that the dependency on $B$ in LSVI-UCB-Batch is tight.

For the future work, we would like to prove the regret lower bound for the rare policy switching model that explicitly depends on the given adaptivity budget $B$.


\section*{Acknowledgments and Disclosure of Funding}

We would like to thank the anonymous reviewers for their helpful comments. Part of this work was done when DZ and QG participated the Theory of Reinforcement Learning program at the Simons Institute for the Theory of Computing in Fall 2020.
DZ and QG are partially supported by the National Science Foundation CAREER Award 1906169, IIS-1904183 and AWS Machine Learning Research Award. The views and conclusions contained in this paper are those of the authors and should not be interpreted as representing any funding agencies.





\bibliography{reference.bib}
\bibliographystyle{ims.bst}

\section*{Checklist}


\begin{enumerate}

\item For all authors...
\begin{enumerate}
  \item Do the main claims made in the abstract and introduction accurately reflect the paper's contributions and scope?
    \answerYes{}
  \item Did you describe the limitations of your work?
    \answerYes{}
  \item Did you discuss any potential negative societal impacts of your work? 
    \answerNA{} The goal of our paper is to develop general algorithms and theoretical analyses for RL with linear function approximation under adaptivity constraints. In this regard, we believe there are no societal impacts because this paper is mainly a theoretical work. 
  \item Have you read the ethics review guidelines and ensured that your paper conforms to them?
    \answerYes{}
\end{enumerate}

\item If you are including theoretical results...
\begin{enumerate}
  \item Did you state the full set of assumptions of all theoretical results?
    \answerYes{}
	\item Did you include complete proofs of all theoretical results?
    \answerYes{}
\end{enumerate}

\item If you ran experiments...
\begin{enumerate}
  \item Did you include the code, data, and instructions needed to reproduce the main experimental results (either in the supplemental material or as a URL)?
    \answerNo{}
  \item Did you specify all the training details (e.g., data splits, hyperparameters, how they were chosen)?
    \answerYes{}
	\item Did you report error bars (e.g., with respect to the random seed after running experiments multiple times)?
    \answerYes{}
	\item Did you include the total amount of compute and the type of resources used (e.g., type of GPUs, internal cluster, or cloud provider)?
    \answerYes{}
\end{enumerate}

\item If you are using existing assets (e.g., code, data, models) or curating/releasing new assets...
\begin{enumerate}
  \item If your work uses existing assets, did you cite the creators?
    \answerNA{} We do not use any existing assets.
  \item Did you mention the license of the assets?
    \answerNA{}
  \item Did you include any new assets either in the supplemental material or as a URL?
    \answerNA{}
  \item Did you discuss whether and how consent was obtained from people whose data you're using/curating?
    \answerNA{}
  \item Did you discuss whether the data you are using/curating contains personally identifiable information or offensive content?
    \answerNA{}
\end{enumerate}

\item If you used crowdsourcing or conducted research with human subjects...
\begin{enumerate}
  \item Did you include the full text of instructions given to participants and screenshots, if applicable?
    \answerNA{} Our work does not involve human subjects.
  \item Did you describe any potential participant risks, with links to Institutional Review Board (IRB) approvals, if applicable?
    \answerNA{}
  \item Did you include the estimated hourly wage paid to participants and the total amount spent on participant compensation?
    \answerNA{}
\end{enumerate}

\end{enumerate}



\newpage
\appendix

\section{Additional Details on the Numerical Experiments}

\subsection{Log-scaled Plot of the Average Regret}

We also provide log-scaled plot of the average regret in Figure \ref{fig2}. We can see that the slope of the average regret curves for our proposed algorithms is similar to that of the fully adaptive LSVI-UCB, all indicating an $\tilde O(1/\sqrt{T})$ scaling.

\begin{figure}[t]
    \begin{subfigure}[LSVI-UCB-Batch]{\label{fig:subfig:2.a}
        \includegraphics[width=0.49\linewidth]{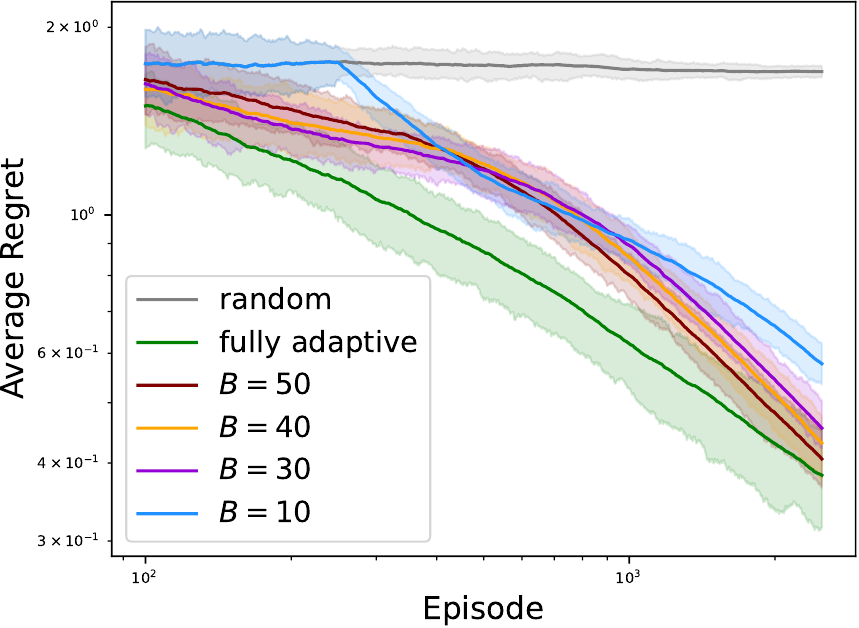}}
    \end{subfigure}
    \begin{subfigure}[LSVI-UCB-RareSwitch]
    {\label{fig:subfig:2.b} 
        \includegraphics[width=0.49\linewidth]{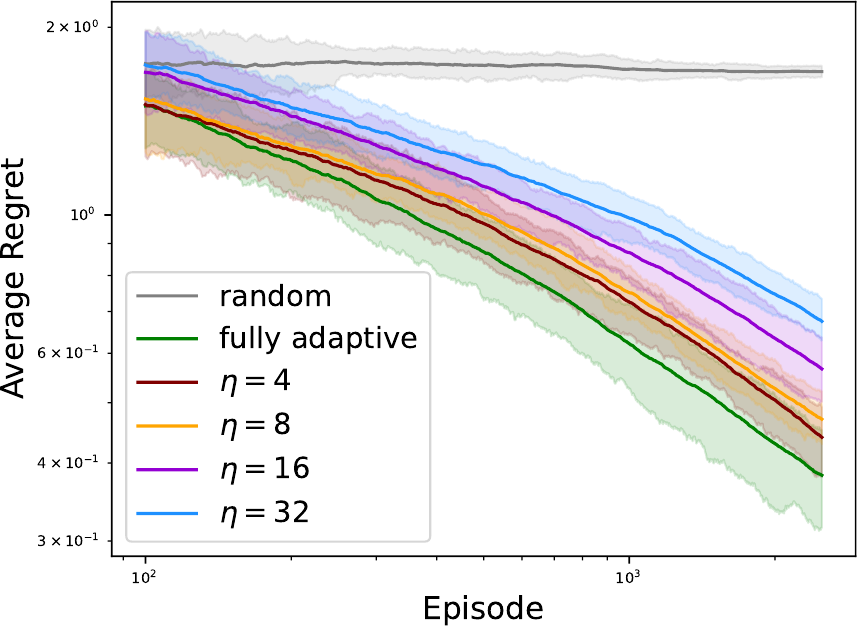}}
    \end{subfigure}
    \caption{Plot of average regret ($\text{Regret}(T)/K$) v.s. the number of episodes in log-scale. 
    The results are averaged over 50 rounds of each algorithm, and the error bars are chosen to be $[20\%,80\%]$ empirical confidence intervals.}
    \label{fig2}
\end{figure}

\subsection{Misspecified Linear MDP}
We also empirically evaluate our algorithms on linear MDP with different levels of misspecification.
In particular, based on the linear MDP instance constructed in Example \ref{eg: mdp1}, we follow the definition of $\zeta$-approximate linear MDP in \citet{jin2020provably}, and consider a corrupted transition given by
\begin{align*}
    \PP_h(s'|0,a) = (1-f(a))\bphi(0,a)^\top \bmu_h(s') + f(a)\ind\{s'=g(a)\}
\end{align*}
where $f:\cA\to[0,\zeta]$, $\zeta\in(0,1)$ and $g:\cA\to\cS$ are unknown.
The two additional functions, $f$ and $g$, can be constructed by random sampling before running the algorithms, and the magnitude of $\zeta\in(0,1)$ characterizes the level of model misspecification.
All the other components of the model and the experiment configurations remain the same as those in Section \ref{sec: experiment}.

\begin{figure}[t]
    \begin{subfigure}[LSVI-UCB-Batch ($B=50$)]{\label{fig:subfig:3.a}
        \includegraphics[width=0.49\linewidth]{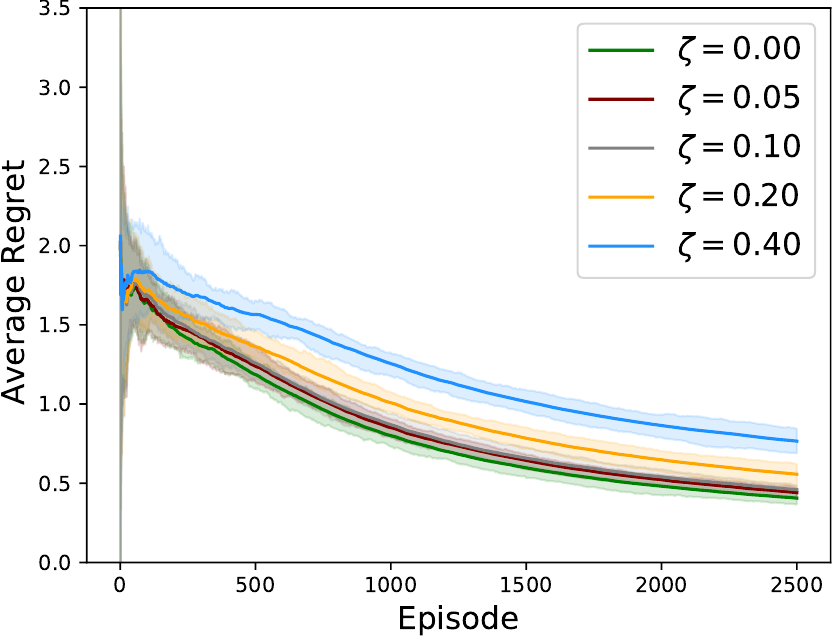}}
    \end{subfigure}
    \begin{subfigure}[LSVI-UCB-RareSwitch ($\eta=8$)]
    {\label{fig:subfig:3.b} 
        \includegraphics[width=0.49\linewidth]{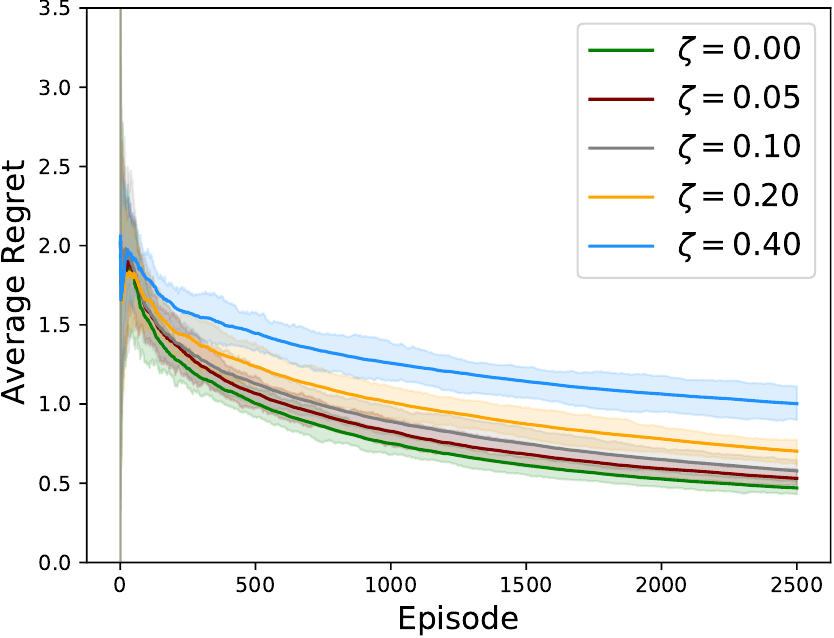}}
    \end{subfigure}
    \caption{Plot of average regret ($\text{Regret}(T)/K$) v.s. the number of episodes for a misspecified linear MDP. 
    The results are averaged over 50 rounds of each algorithm, and the error bars are chosen to be $[20\%,80\%]$ empirical confidence intervals.}
    \label{fig3}
\end{figure}

Under this misspecified model with levels $\zeta=0.05, 0.1, 0.2, 0.4$, we run LSVI-UCB-Batch with $B=50$ and LSVI-UCB-RareSwitch with $\eta=8$ respectively.
We plot the average regret of the algorithms in Figure \ref{fig3}.
We can see that our algorithms can still achieve a reasonably good performance under considerable levels of model misspecification.

\section{Proofs of Theorem \ref{thm: regret budget1}}
In this section we prove Theorem \ref{thm: regret budget1}

For simplicity, we use $b_k$ to denote the batch $t_b$ satisfying $t_b \leq k <t_{b+1}$. Let $\Gamma_h^k(\cdot, \cdot)$ be $\beta\cdot[\bphi(\cdot,\cdot)^\top(\bLambda_h^k)^{-1} \bphi(\cdot,\cdot)]^{1/2}$ for any $h \in [H], k\in[K]$. First, we need the following lemma which gives $\text{Regret}(T)$ a high probability upper bound that depends on the summation of bonuses. 
\begin{lemma}\label{lemma:rough_bound}
With probability at least $1-\delta$, the total regret of Algorithm \ref{alg: batch} satisfies
\begin{align}
    \text{Regret}(T)&\leq \sum_{k=1}^K\sum_{h=1}^H\min\Big\{H, 2\Gamma_h^{b_k}(s_h^k,a_h^k)\Big\} + 2H\sqrt{T\log\rbr{\frac{2dT}{\delta}}}.\notag
\end{align}
\end{lemma}
Lemma \ref{lemma:rough_bound} suggests that in order to bound the total regret, it suffices to bound the summation of the `delayed' bonuses $\Gamma_h^{b_k}(s_h^k,a_h^k)$, in contrast to the per-episode bonuses $\Gamma_h^k(s_h^k,a_h^k)$ for all $k\in [K]$. The superscript $b_k$ suggests that instead of using all the information up to the current episode $k$, Algorithm \ref{alg: batch} can only use the information before the current batch $b_k$ due to its batch learning nature. How to control the error induced by batch learning is the main difficulty in our analysis. To tackle this difficulty, we first need an upper bound for the summation of per-episode bonuses $\Gamma_h^k(s_h^k,a_h^k)$. 
\begin{lemma}\label{lemma: bonus bound}
Let $\beta$ be selected as Theorem \ref{thm: regret budget1} suggests. Then the summation of all the per-episode bonuses is bounded by
\begin{align*}
    \sum_{k=1}^K\sum_{h=1}^H\Gamma_h^k(s_h^k,a_h^k)\leq \beta \sqrt{2dHT\log\bigg(\frac{T}{dH}+1\bigg)}.
\end{align*}
\end{lemma}
It is worth noting that the per-episode bonuses are not generated from our algorithm, but instead are some virtual terms that we introduce to facilitate our analysis.
Equipped with Lemma \ref{lemma: bonus bound}, we only need to bound the difference between delayed bonuses and per-episode bonuses. We consider all the indices $(k,h) \in [K]\times[H]$. The next lemma suggests that considering the ratio between delayed bonuses and per-episode bonuses, the `bad' indices, where the ratio is large, only appear few times. This is also the key lemma of our analysis.  
\begin{lemma}\label{lemma:except}
Define the set $\cC$ as follows
\begin{align}
    \cC = \{(k,h): \Gamma_h^{b_k}(s_h^k,a_h^k)/\Gamma_h^{k}(s_h^k,a_h^k)>2\},\notag
\end{align}
then we have $|\cC| \leq dHK\log(K/d+1)/(2B\log 2)$. 
\end{lemma}
With all the above lemmas, we now begin to prove our main theorem. 
\begin{proof}[Proof of Theorem \ref{thm: regret budget1}]
Suppose the event defined in Lemma \ref{lemma:rough_bound} holds. Then by Lemma \ref{lemma:rough_bound} we have that
\begin{align}
    \text{Regret}(T)&\leq\underbrace{\sum_{h=1}^H\sum_{k=1}^K\min\Big\{H, 2\Gamma_h^{b_k}(s_h^k,a_h^k)\Big\}}_{I} + 2H\sqrt{T\log\rbr{\frac{2dT}{\delta}}}\label{eq:thm1}
\end{align}
holds with probability at least $1-\delta$. 
Next, we are going to bound $I$. Let $\cC$ be the set defined in Lemma \ref{lemma:except}. Then we have
\begin{align}
    I  &= \sum_{(k,h) \in \cC} \min\Big\{H, 2\Gamma_h^{b_k}(s_h^k,a_h^k)\Big\} + \sum_{(k,h) \notin \cC} \min\Big\{H, 2\Gamma_h^{b_k}(s_h^k,a_h^k)\Big\} \notag \\
    & \leq H|\cC| + 4\sum_{(k,h) \notin \cC} \Gamma_h^{k}(s_h^k,a_h^k) \notag \\
    & \leq H|\cC| + 4\sum_{h=1}^H\sum_{k=1}^K\Gamma_h^{k}(s_h^k,a_h^k), \label{eq:thm2}
\end{align}
where the first inequality holds due to the definition of $\cC$, and the second one holds trivially. Therefore, substituting \eqref{eq:thm2} into \eqref{eq:thm1}, the regret can be bounded by 
\begin{align}
    \text{Regret}(T)&\leq 2H\sqrt{T\log\rbr{\frac{2dT}{\delta}}}+H|\cC| + 4\sum_{h=1}^H\sum_{k=1}^K\Gamma_h^{k}(s_h^k,a_h^k)\notag \\
    & \leq 2H\sqrt{T\log\rbr{\frac{2dT}{\delta}}} + \frac{dHT}{2B\log 2}\log\rbr{\frac{T}{dH}+1} \\
    &\qquad + 4c\sqrt{2d^3H^3T\log\rbr{\frac{2dT}{\delta}}\log\bigg(\frac{T}{dH}+1\bigg)},\notag
\end{align}
where the second inequality holds due to Lemmas \ref{lemma: bonus bound} and \ref{lemma:except} and the fact that $T=KH$. This completes the proof.
\end{proof}

\subsection{Proof of Lemma \ref{lemma:rough_bound}}
The following two lemmas in \citet{jin2020provably} characterize the quality of the estimates given by the LSVI-UCB-type algorithms.

\begin{lemma}[Lemma B.5, \citealt{jin2020provably}]\label{lemma: ucb}
With probability at least $1-\delta$, we have $Q_h^k(s,a)\geq Q_h^*(s,a)$ for all $(s,a,h,k)\in\cS\times\cA\times[H]\times[K]$.
\end{lemma}

\begin{lemma}[Lemma B.4, \citealt{jin2020provably}]\label{lemma: bonus}
There exists some constant $c$ such that if we set $\beta=cdH\sqrt{\log(dT/\delta)}$, then for any fixed policy $\pi$ we have for all $(s,a,h,k)\in\cS\times\cA\times[H]\times[K]$ that
\begin{align*}
    \abr{\bphi(s,a)^\top(\wb_h^{b_k}-\wb_h^\pi)-\rbr{\PP_h(V_{h+1}^{b_k}-V_{h+1}^\pi)}(s,a)}&\leq \beta \sqrt{\bphi(s,a)^\top(\bLambda_h^{b_k})^{-1}\bphi(s,a)}
\end{align*}
with probability at least $1-\delta$.
\end{lemma}

\begin{proof}[Proof of Lemma \ref{lemma:rough_bound}]
By Lemma \ref{lemma: ucb}, we have $Q_h^k(s,a)\geq Q_h^*(s,a)$ for all $(s,a,h,k)\in\cS\times\cA\times[H]\times[K]$ on some event $\cE$ such that $\PP(\cE)\geq 1-\delta/2$. In the following argument, all statements would be conditioned on the event $\cE$. Then by the definition of $V_1^k$ we know that $V_1^k(s)=\max_{a\in\cA} Q_1^k(s,a)\geq \max_{a\in\cA} Q_1^*(s,a)=V_1^*(s)$ for all $(s,k)\in\cS\times[K]$. Therefore, we have
\begin{align}
    \text{Regret}(T)&=\sum_{k=1}^K\sbr{V_1^*(s_1^k)-V_1^{\pi^k}(s_1^k)}\leq \sum_{k=1}^K \sbr{V_1^{k}(s_1^k) - V_1^{\pi^k}(s_1^k)} = \sum_{k=1}^K \sbr{V_1^{b_k}(s_1^k) - V_1^{\pi^k}(s_1^k)}.\notag
\end{align}
Note that 
\begin{align*}
    V_{h}^{b_k}(s_h^k) - V_h^{\pi^k}(s_h^k)&=Q_h^{b_k}(s_h^k,a_h^k) - Q_h^{\pi^k}(s_h^k,a_h^k),
\end{align*}
which together with the definition of $Q_h^{b_k}$ and Lemma \ref{lemma: bonus} implies that 
\begin{align*}
    V_{h}^{b_k}(s_h^k) - V_h^{\pi^k}(s_h^k)&\leq \bphi(s_h^k,a_h^k)^\top\wb_h^{b_k}-\bphi(s_h^k,a_h^k)^\top\wb_h^{\pi^k}+\Gamma_h^{b_k}(s_h^k,a_h^k)\\
    &\leq \sbr{\PP_h\rbr{V_{h+1}^{b_k}-V_{h+1}^{\pi^k}}}(s_h^k,a_h^k) + 2\Gamma_h^{b_k}(s_h^k,a_h^k), 
\end{align*}
where the first inequality holds due to the algorithm design, the second one holds due to Lemma~\ref{lemma: bonus}. 
Meanwhile, notice that $0 \leq V_{h}^{b_k}(s_h^k) - V_h^*(s_h^k) \leq V_{h}^{b_k}(s_h^k) - V_h^{\pi^k}(s_h^k) \leq H$, then we have
\begin{align}
    V_{h}^{b_k}(s_h^k) - V_h^{\pi^k}(s_h^k)&\leq\min\Big\{H, \sbr{\PP_h\rbr{V_{h+1}^{b_k}-V_{h+1}^{\pi^k}}}(s_h^k,a_h^k) + 2\Gamma_h^{b_k}(s_h^k,a_h^k)\Big\}\notag \\
    & \leq\sbr{\PP_h\rbr{V_{h+1}^{b_k}-V_{h+1}^{\pi^k}}}(s_h^k,a_h^k) + \min\Big\{H,  2\Gamma_h^{b_k}(s_h^k,a_h^k)\Big\}\notag\\
    &=V_{h+1}^{b_k}(s_{h+1}^k)-V_{h+1}^{\pi^k}(s_{h+1}^k) + \min\big\{H, 2\Gamma_h^{b_k}(s_h^k,a_h^k)\big\}\notag\\
    &\qquad + \sbr{\PP_h\rbr{V_{h+1}^{b_k}-V_{h+1}^{\pi^k}}}(s_h^k,a_h^k) - \rbr{V_{h+1}^{b_k}(s_{h+1}^k)-V_{h+1}^{\pi^k}(s_{h+1}^k)},\notag
\end{align}
where the second inequality holds since $V_{h+1}^{b_k}-V_{h+1}^{\pi^k}\geq 0$. Recursively expand the above inequality, and we have
\begin{align*}
    V_1^{b_k}(s_1^k)-V_1^{\pi^k}(s_1^k)&=\sum_{h=1}^H\cbr{\sbr{\PP_h\rbr{V_{h+1}^{b_k}-V_{h+1}^{\pi^k}}}(s_h^k,a_h^k) - \rbr{V_{h+1}^{b_k}(s_{h+1}^k)-V_{h+1}^{\pi^k}(s_{h+1}^k)}}\\
    &\qquad + \sum_{h=1}^H\min\Big\{H, 2\Gamma_h^{b_k}(s_h^k,a_h^k)\Big\}.
\end{align*}
Therefore, the total regret can be bounded as follows
\begin{align*}
    \text{Regret}(T)&\leq \sum_{k=1}^K\sum_{h=1}^H\cbr{\sbr{\PP_h\rbr{V_{h+1}^{b_k}-V_{h+1}^{\pi^k}}}(s_h^k,a_h^k) - \rbr{V_{h+1}^{b_k}-V_{h+1}^{\pi^k}}(s_{h+1}^k)}\\
    &\qquad + \sum_{k=1}^K\sum_{h=1}^H\min\Big\{H, 2\Gamma_h^{b_k}(s_h^k,a_h^k)\Big\}.
\end{align*}
Note that conditional on $\cF_{k,h,1}$, $V_{h+1}^{b_k}$ and $V_{h+1}^{\pi^k}$ are both deterministic, while $s_{h+1}^k$ follows the distribution $\PP_h(\cdot|s_h^k,a_h^k)$. Therefore, the first term on the RHS is a sum of a martingale difference sequence such that each summand has absolute value at most $2H$. Applying Azuma-Hoeffding inequaliy yields
\begin{align*}
    \sum_{k=1}^K\sum_{h=1}^H\cbr{\sbr{\PP_h\rbr{V_{h+1}^{b_k}-V_{h+1}^{\pi^k}}}(s_h^k,a_h^k) - \rbr{V_{h+1}^{b_k}-V_{h+1}^{\pi^k}}(s_{h+1}^k)}\leq 2H\sqrt{T\log\rbr{\frac{2dT}{\delta}}},
\end{align*}
with probability at least $1-\delta/2$.
By a union bound over the event $\cE$ and the convergence of the martingale, with probability at least $1-\delta$, we have
\begin{align}
    \text{Regret}(T)\leq 2H\sqrt{T\log\rbr{\frac{2dT}{\delta}}}+\sum_{k=1}^K\sum_{h=1}^H\min\Big\{H, 2\Gamma_h^{b_k}(s_h^k,a_h^k)\Big\}.\notag
\end{align}
\end{proof}

\subsection{Proof of Lemma \ref{lemma: bonus bound}}
We need the following lemma to bound the sum of the bonus terms.
\begin{lemma}[Lemma 11, \citealt{abbasi2011improved}]\label{lemma: elliptical potential}
Let $\{\bphi_t\}_{t=1}^{\infty}$ be an $\RR^d-$valued sequence. Meanwhile, let $\bLambda_0\in\RR^{d\times d}$ be a positive-definite matrix and $\bLambda_t=\bLambda_0+\sum_{i=1}^{t-1}\bphi_i\bphi_i^\top$. It holds for any $t\in\ZZ_+$ that
\begin{align*}
    \sum_{i=1}^t \min\{1, \bphi_i^\top\bLambda_i^{-1}\bphi_i\}\leq 2\log\left(\frac{\det(\bLambda_{t+1})}{\det(\bLambda_1)}\right).
\end{align*}
Moreover, assuming that $\|\bphi_i\|_2\leq 1$ for all $i\in\ZZ_+$ and $\lambda_{\min}(\bLambda_0)\geq 1$, it holds for any $t\in\ZZ_+$ that
\begin{align*}
    \log\left(\frac{\det(\bLambda_{t+1})}{\det(\bLambda_1)}\right)\leq \sum_{i=1}^t\bphi_i^\top\bLambda_i^{-1}\bphi_i\leq 2\log\left(\frac{\det(\bLambda_{t+1})}{\det(\bLambda_1)}\right).
\end{align*}
\end{lemma}
\begin{proof}[Proof of Lemma \ref{lemma: bonus bound}]
We can bound the summation of $\Gamma_h^k(s_h^k, a_h^k)$ as follows:
\begin{align}
    \sum_{h=1}^H\sum_{k=1}^K \Gamma_h^k(s_h^k, a_h^k) \leq \sum_{h=1}^H\sqrt{K\cdot \sum_{k=1}^K [\Gamma_h^k(s_h^k, a_h^k)]^2} = \beta\sqrt{K}\sum_{h=1}^H\sqrt{\sum_{k=1}^K \bphi(s_h^k, a_h^k)^\top [\bLambda_h^k]^{-1}\bphi(s_h^k, a_h^k)},\notag
\end{align}
where the inequality holds due to Cauchy-Schwarz inequality. Furthermore, by Lemma \ref{lemma: elliptical potential}, we have
\begin{align}
    \sum_{k=1}^K \bphi(s_h^k, a_h^k)^\top [\bLambda_h^k]^{-1}\bphi(s_h^k, a_h^k) \leq 2\log\bigg(\frac{\det\bLambda_h^{K+1}}{\det\bLambda_h^{1}}\bigg) \leq 2d\log(K/d+1),\notag
\end{align}
where the second inequality holds due to Lemma \ref{lemma: det bound}. That finishes our proof. 
\end{proof}

\subsection{Proof of Lemma \ref{lemma:except}}

\begin{proof}[Proof of Lemma \ref{lemma:except}]
First, let $\cC_h$ denote the indices $k$ where $(k,h) \in \cC$, then we have $|\cC| = \sum_{h=1}^H|\cC_h|$. 
Next we bound $|\cC_h|$ for each $h$. For each $k \in \cC_h$, suppose $t_b \leq k<t_{b+1}$, then we have $b_k = t_b$ and
\begin{align}
    \log\det(\bLambda_h^{t_{b+1}}) - \log\det(\bLambda_h^{t_{b}}) \geq \log\det(\bLambda_h^k) - \log\det(\bLambda_h^{b_k}) \geq 2\log(\Gamma_h^{b_k}(s_h^k,a_h^k)/\Gamma_h^{k}(s_h^k,a_h^k)) >2\log 2,\notag
\end{align}
where the first inequality holds since $\bLambda_h^{t_{b+1}}\succeq \bLambda_h^k$, the second inequality holds due to Lemma \ref{lemma: quadratic form}, the third one holds due to the definition of $\cC_h$. Thus, let $\hat\cC_h$ denote the set
\begin{align}
    \hat\cC_h = \{b \in [B]: \log\det(\bLambda_h^{t_{b+1}}) - \log\det(\bLambda_h^{t_{b}})>2\log 2\}, \notag
\end{align}
we have $|\cC_h| \leq \lfloor K/B\rfloor\cdot |\hat\cC_h|$. In the following we bound $|\hat\cC_h|$. Now we consider the sequence $\{\log\det(\bLambda_h^{t_{b+1}}) - \log\det(\bLambda_h^{t_{b}})\}$. It is easy to see $\log\det(\bLambda_h^{t_{b+1}}) - \log\det(\bLambda_h^{t_{b}}) \geq 0$, therefore
\begin{align}
    2\log 2|\hat\cC_h| \leq \sum_{b \in \hat\cC_h} [\log\det(\bLambda_h^{t_{b+1}}) - \log\det(\bLambda_h^{t_{b}})] \leq \sum_{b=1}^B [\log\det(\bLambda_h^{t_{b+1}}) - \log\det(\bLambda_h^{t_{b}})].\label{eq:zhou:777}
\end{align}
Meanwhile, we have
\begin{align}
    \sum_{b=1}^B [\log\det(\bLambda_h^{t_{b+1}}) - \log\det(\bLambda_h^{t_{b}})] = \log\det(\bLambda_h^{t_{B+1}}) = \log\det(\bLambda_h^{K+1})\leq d\log (K/d+1),\label{eq:zhou:888}
\end{align}
where the last inequality holds due to Lemma \ref{lemma: det bound}. 
Therefore, \eqref{eq:zhou:777} and \eqref{eq:zhou:888} suggest that $|\hat\cC_h| \leq d\log(K/d+1)/(2\log 2)$. Finally, we bound $|\cC|$ as follows, which ends our proof.
\begin{align}
    |\cC| = \sum_{h=1}^H |\cC_h| \leq \sum_{h=1}^H K/B\cdot |\hat\cC_h| \leq dHK\log(K/d+1)/(2B\log 2).\notag
\end{align}
\end{proof}

\section{Proof of Theorem \ref{thm: regret rare switch}}

Now we provide the proof of Theorem \ref{thm: regret rare switch}. We continue to use the notions that have been introduced in Section \ref{sec:batch}.
We first give an upper bound on the determinant of $\bLambda_h^k$.
\begin{lemma}\label{lemma: det bound}
Let $\{\bLambda_h^k,(k,h)\in[K]\times[H]\}$ be as defined in Algorithms \ref{alg: batch} and \ref{alg: rare switch}. Then for all $h\in[H]$ and $k \in [K]$, we have $\det(\bLambda_h^k)\leq (\lambda + (k-1)/d)^d$. 
\end{lemma}
\begin{proof}
Note that
\begin{align*}
    \tr(\bLambda_h^k) &= \tr(\lambda\Ib_d) + \sum_{\tau=1}^{k-1} \tr\rbr{\bphi(s_h^\tau,a_h^\tau)\bphi(s_h^\tau,a_h^\tau)^\top}=\lambda d + \sum_{\tau=1}^{k-1} \|\bphi(s_h^\tau,a_h^\tau)\|_2^2\leq \lambda d + k-1,
\end{align*}
where the inequality follows from the assumption that $\|\bphi(s,a)\|_2\leq 1$ for all $(s,a)\in\cS\times\cA$. Since $\bLambda_h^k$ is positive semi-definite, by inequality of arithmetic and geometric means, we have
\begin{align*}
    \det(\bLambda_h^k)\leq \rbr{\frac{\tr(\bLambda_h^k)}{d}}^d\leq\rbr{\lambda+\frac{k-1}{d}}^d.
\end{align*}
This finishes the proof.
\end{proof}

Next lemma provides a determinant-based upper bound for the ratio between the norms $\|\cdot\|_{\Ab}$ and $\|\cdot\|_{\Bb}$, where $\Ab \succeq \Bb$.

\begin{lemma}[Lemma 12, \citealt{abbasi2011improved}]\label{lemma: quadratic form}
Suppose $\Ab,\Bb\in\RR^{d\times d}$ are two positive definite matrices satisfying that $\Ab\succeq \Bb$, then for any $\xb\in\RR^d$, we have $\|\xb\|_{\Ab}\leq \|\xb\|_{\Bb}\cdot\sqrt{\det(\Ab)/\det(\Bb)}$.
\end{lemma}

The switching cost of Algorithm \ref{alg: rare switch} is characterized in the following lemma.

\begin{lemma}\label{lemma: switching cost}
For any $\eta>1$ and $\lambda>0$, the global switching cost of Algorithm \ref{alg: rare switch} is bounded by
\begin{align*}
    N_{\text{switch}}\leq \frac{dH}{\log\eta}\log\left(1+\frac{K}{\lambda d}\right).
\end{align*}
\end{lemma}
\begin{proof}
Let $\{k_1,k_2,\cdots,k_{N_{\text{switch}}}\}$ be the episodes where the algorithm updates the policy, and we also define $k_0=0$. Then by the determinant-based criterion (Line \ref{alg rare switch: log det}), for each $i\in[N_{\text{switch}}]$ there exists at least one $h\in[H]$ such that
\begin{align*}
    \det(\bLambda_h^{k_i}) > \eta\cdot \det(\bLambda_h^{k_{i-1}}).
\end{align*}
By the definition of $\bLambda_h^k$ (Line \ref{alg rare switch: Lambda^k}), we know that $\bLambda_h^{j_1}\succeq\bLambda_h^{j_2}$ for all $j_1\geq j_2$ and $h\in[H]$. Thus we further have
\begin{align*}
    \prod_{h=1}^H\det(\bLambda_h^{k_i})> \eta\cdot\prod_{h=1}^H\det(\bLambda_h^{k_{i-1}}).
\end{align*}

Applying the above inequality for all $i\in[N_{\text{switch}}]$ yields
\begin{align*}
    \prod_{h=1}^H \det\left(\bLambda_h^{k_{N_{\text{switch}}}}\right) > \eta^{N_{\text{switch}}}\cdot \prod_{h=1}^H \det(\bLambda_h^0)=\eta^{N_{\text{switch}}}\lambda^{dH},
\end{align*}
as we initialize $\bLambda_h^0$ to be $\lambda\Ib_d$. While by Lemma \ref{lemma: det bound}, we have
\begin{align*}
    \prod_{h=1}^H \det\left(\bLambda_h^{k_{N_{\text{switch}}}}\right)\leq \prod_{h=1}^H\det(\bLambda_h^K)\leq \left(\lambda+\frac{K}{d}\right)^{dH}.
\end{align*}
Therefore, combining the above two inequalities, we obtain that
\begin{align*}
    N_{\text{switch}}\leq \frac{dH}{\log\eta}\log\rbr{1+\frac{K}{\lambda d}}.
\end{align*}
This completes the proof.
\end{proof}
We now begin to prove our main theorem. 
\begin{proof}[Proof of Theorem \ref{thm: regret rare switch}]
First, substituting the choice of $\eta$ and $\lambda=1$ into the bound in Lemma \ref{lemma: switching cost} yields that $N_{\text{switch}} \leq B$. 

Next, we bound the regret of Algorithm \ref{alg: rare switch}. The result of Lemma \ref{lemma:rough_bound} still holds here, thus it suffices to bound the summation of the bonus terms $\Gamma_h^{b_k}(s_h^k,a_h^k)$. Note that $b_k\leq k$, and thus $\bLambda_h^k\succeq\bLambda_h^{b_k}$ for all $(h,k)\in[H]\times[K]$. Then by Lemma \ref{lemma: quadratic form} we have
\begin{align}
    \frac{\Gamma_h^{b_k}(s_h^k,a_h^k)}{\Gamma_h^k(s_h^k,a_h^k)}\leq\sqrt{\frac{\det(\bLambda_h^k)}{\det(\bLambda_h^{b_k})}} \leq \sqrt{\eta}
\end{align}
for all $(h,k)\in[H]\times[K]$, where the second inequality holds due to the algorithm design. Hence, we have 
\begin{align*}
    \sum_{k=1}^K\sum_{h=1}^H \Gamma_h^{b_k}(s_h^k,a_h^k)&\leq\sqrt{\eta}\sum_{k=1}^K\sum_{h=1}^H\Gamma_h^k(s_h^k,a_h^k)\leq \beta\sqrt{2\eta dHT\log\rbr{\frac{T}{dH}+1}},
\end{align*}
where the second inequality follows from Lemma \ref{lemma: bonus bound}. Therefore, we conclude by Lemma \ref{lemma:rough_bound} that
\begin{align}
    \text{Regret}(T)&\leq2c\sqrt{2\eta d^3H^3T\log\left(\frac{T}{dH}+1\right)\log\rbr{\frac{2dT}{\delta}}} + 2H\sqrt{T\log\rbr{\frac{2dT}{\delta}}}\label{eq:zhou:1}
\end{align}
holds with probability at least $1-\delta$. Finally, substituting the choice of $\eta$ into \eqref{eq:zhou:1} finishes our proof. 
\end{proof}

\section{Proofs of Theorem \ref{thm:lower1}}
In this section, we prove the lower bound for the batch learning model. 


\begin{proof}[Proof of Theorem \ref{thm:lower1}]

We prove the $\Omega(dH\sqrt{T})$ and $\Omega(dHT/B)$ lower bounds separately. The first term has been proved in Theorem 5.6, \citep{zhou2020nearly}. In the remaining of this proof, we prove the second term. 
We consider a class of MDPs parameterized by $\bgamma \in \Gamma \subset\RR^{2dH}$, where $\Gamma$ is defined as follows
\begin{align}
    \Gamma = \big\{(\bbb_{1,1}^\top,\cdots, \bbb_{H, d}^\top)^\top: \bbb_{i,j}\in\{(0, 1)^\top, (1,0)^\top\}\big\}.\notag
\end{align}
The MDP is defined as follows. The states space $\cS$ consist of has $d+1$ states $x_0,\cdots, x_d$, and the action space $\cA$ contains two actions $\ab_1 = (0,1)^\top, \ab_2 = (1,0)^\top$. For any $\bgamma = (\bbb_{1,1}^\top,\cdots, \bbb_{H, d}^\top)^\top$, the feature mapping is defined as 
\begin{align*}
\bphi(x_0, \ab_j) = (1, \underbrace{0,\cdots, 0}_{2d})^\top,\qquad \bphi(x_i, \ab_j) = (1, \underbrace{0, \cdots, 0}_{2i-2}, \ab_j^\top,\underbrace{0, \cdots, 0}_{2d-2i} )^\top\in\RR^{2d+1}
\end{align*}
for every $i\in[d]$ and $j\in\{1,2\}$. We further define the vector-valued measures as
\begin{align*}
    \bmu_h^{\bgamma}(x_0) = (1, -\bbb_{h,1}^\top,\cdots, -\bbb_{h, d}^\top)^\top ,\qquad \bmu_h^{\bgamma}(x_i) = (\underbrace{0, \cdots, 0}_{2i-1}, \bbb_{h, i}^\top,\underbrace{0, \cdots, 0}_{2d-2i} )^\top
\end{align*}
for every $i\in[d]$, $j\in\{1,2\}$ and $h\in[H]$. Finally, for each $h\in[H]$, we define
\begin{align*}
    \btheta_h = (0, \underbrace{1,\cdots, 1}_{2d})^\top \in\RR^{2d+1}.
\end{align*}

Thereby, for each $h\in[H]$, the transition $\PP_h^{\bgamma}$ is defined as $\PP_h^{\bgamma}(s'|s, \ab) = \la\bphi(s, \ab), \bmu_h^{\bgamma}(s') \ra$, and the reward function is $r_h(s, \ab) = \la \bphi(s, \ab), \btheta_h\ra$ for all $(s,\ab)\in\cS\times\cA$. It is straightforward to see that the reward satisfies $r_h(x_0, \ab) = 0$ and $r_h(x_i, \ab) = 1$ for $i \in[d]$ and all $\ab\in\cA$. In addition, the starting state can be $x_0$ or $x_i$. 

Based on the above definition, we have the following transition dynamic:
\begin{itemize}
    \item $x_0$ is an absorbing state.
    \item For any $i\in[d]$, $x_i$ can only transit to $x_0$ or $x_i$.
    \item For any episode starting from $x_0$, there is no regret.
    \item For any episode starting from some $x_i$ with $i\in[d]$, suppose $h$ is the first stage where the agent did not choose the "right" action $\ab = \bbb_{h,i}$, then the regret for this episode is $H-h$. 
\end{itemize}

Now we show that for any deterministic algorithm\footnote{The lower bound of random algorithms is lower bounded by the lower bound of deterministic algorithms according to Yao's minimax principle. }, there exists a $\bgamma \in \Gamma$ such that the regret is lower bounded by $dHT/B$.
Suppose $1 = t_1<\cdots<t_{B+1} = K+1$. 
We can treat all episodes in the same batch as copies of one episode, because all actions taken by the agent, transitions and rewards are the same. When $B \geq dH$, there exists $\cC =  \{c_{1,1},\cdots, c_{H,d}\}\subset[B]$ with $|\cC| = dH$ such that
\begin{align}
    \sum_{h \in [H]} \sum_{j \in [d]}(t_{c_{h,j}+1} - t_{c_{h,j}}) \geq \frac{dHK}{B}.\notag
\end{align}
For simplicity, we denote the $i$-th batch as the collection of episodes $\{t_i, \cdots, t_{i+1}-1\}$. 
Now we carefully pick the starting state $s_0^i$ for the episodes in the $i$-th batch. 
\begin{itemize}
    \item For any batch whose starting episode does not belong to $\cC$, we set the starting states of the episodes in this batch as $x_0$. In other words, for $i \notin \cC$, we set $s_0^{t_i}=\cdots =s_0^{t_{i+1}-1} =  x_0$.
    \item For any batch whose starting episode lies in $\cC$, 
    for $i  = c_{h,j}\in \cC$, we set $s_0^{t_{c_{h, j}}}=\cdots = s_0^{t_{c_{h, j}+1}-1} = x_{j}$.
\end{itemize}
 We consider the regret over batches $c_{1,i},\cdots, c_{H,i}$. Since the algorithm, transition and reward are all deterministic, then the environment can predict the agent's selection. Specifically, suppose the agent will always take action $\ab$ at $h$-th stage in the episodes belonging to the $c_{h,j}$-th batch, where $h \leq H/2$. Then the environment selects $\bbb_{h,j}$ as $(1,1)^\top - \ab$, i.e., the other action. Therefore, the agent will always pick the ``wrong" action when she firstly visits state $x_j$ at $h$-th stage, which occurs at least $H-h \geq H/2$ regret. Moreover, since for the batch learning model, all the actions are decided at the beginning of each batch, then the $H/2$ regret will last $(t_{c_{h,j}+1} - t_{c_{h,j}})$ episodes. Taking the summation, we have
\begin{align}
    \text{Regret}(T) \geq \frac{H}{2}\cdot \sum_{h \in [H]} \sum_{j \in [d]}(t_{c_{h,j}+1} - t_{c_{h,j}}) \geq \frac{dHT}{2B}.\notag
\end{align}
Finally, replacing $d$ by $(d-1)/2$, we can convert our feature mapping from a $(2d+1)$-dimensional vector to a $d$-dimensional vector and complete the proof. 
\end{proof}
\end{document}